\documentclass[graybox]{svmult}

\usepackage{mathptmx}       
\usepackage{helvet}         
\usepackage{courier}        
\usepackage{type1cm}        
\usepackage{makeidx}         
\usepackage{graphicx}        
\usepackage{multicol}        
\usepackage[bottom]{footmisc}

\usepackage{subfigure}
\usepackage[lowtilde]{url}
\urldef{\myself}\url{http://robotics.usc.edu/%7egeoff/iser10/}
\usepackage{colonequals}
\usepackage{amsmath}
\usepackage{amssymb}

\usepackage{algorithm}
\usepackage{algorithmic}

\newcommand{\argmax}{\operatornamewithlimits{argmax}}
\newcommand{\argmin}{\operatornamewithlimits{argmin}}

\usepackage[numbers]{natbib}

\makeindex             

\begin{document}

\title*{Active Classification: Theory and Application to Underwater Inspection}
\author{Geoffrey A. Hollinger, Urbashi Mitra, and Gaurav S. Sukhatme}
\institute{Geoffrey A. Hollinger and Gaurav S. Sukhatme \at Computer Science Department, Viterbi School of Engineering, University of Southern California, Los Angeles, CA 90089, \email{{gahollin,gaurav}@usc.edu}
\and Urbashi Mitra \at Electrical Engineering Department, Viterbi School of Engineering, University of Southern California, Los Angeles, CA 90089, \email{ubli@usc.edu}}
%
%
\maketitle

\abstract{We discuss the problem in which an autonomous vehicle must classify an object based on multiple views.  We focus on the active classification setting, where the vehicle controls which views to select to best perform the classification. The problem is formulated as an extension to Bayesian active learning, and we show connections to recent theoretical guarantees in this area. We formally analyze the benefit of acting adaptively as new information becomes available. The analysis leads to a probabilistic algorithm for determining the best views to observe based on information theoretic costs. We validate our approach in two ways, both related to underwater inspection: 3D polyhedra recognition in synthetic depth maps and ship hull inspection with imaging sonar.  These tasks encompass both the planning and recognition aspects of the active classification problem. The results demonstrate that actively planning for informative views can reduce the number of necessary views by up to 80\% when compared to passive methods.}

\section{Introduction}
\label{sect:intro}

Consider the following scenario, which occurs when observing an environment with an underwater vehicle: given a playback of imaging sonar data from the vehicle, the task is to determine which frames contain objects of interest (e.g., mines~\citep{williams09TIP}, explosives, ship wreckage, enemy submarines, marine life~\citep{steinberg10IROS}, etc.). We will refer to these problems as \textit{underwater inspection}, since an object is being inspected to determine its nature. We are interested in utilizing sensor data, such as depth map information, to determine the nature of a potential object of interest. Such problems are typically formulated as \textit{passive} classification, where some data are given, and the goal is to determine the nature of this data. 

While passive classification problems are challenging in themselves, what is often overlooked is that robotic applications allow for \textit{active} decision making. In other words, an autonomous vehicle performing a classification task has control over how it views the environment. The vehicle could change its position, modify parameters on its sensor, or even manipulate the environment to improve its view. For instance, it may be difficult to determine the nature of an object when viewed from the top (due to lack of training data, lack of salient features, occlusions, etc.), but the same object may be easy to identify when viewed from the side. As an example, Figure~\ref{fig:mine} shows an explosive device placed on a ship's hull viewed from two different angles with imaging sonar. The explosive is easier to identify when viewed from the side (left image) versus from above (right image) due to the reflective qualities of its material.

\begin{figure}[tb]
  \centering
  \subfigure{\includegraphics[width=1.4in]{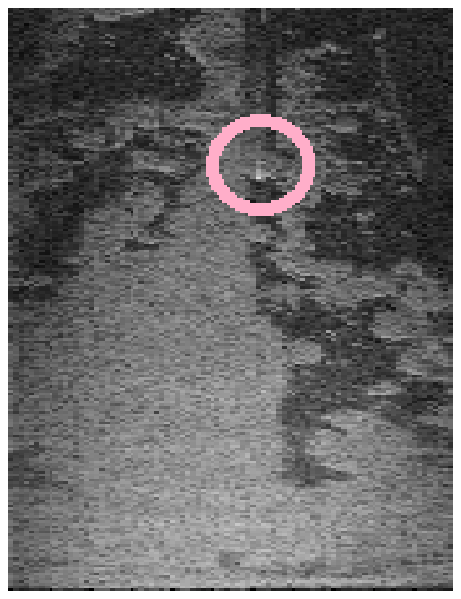}}
  \subfigure{\includegraphics[width=1.4in]{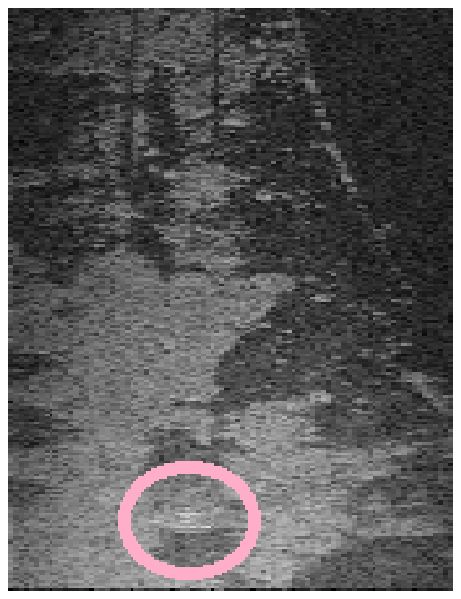}}
  \caption{An explosive device (circled) placed on a ship hull viewed using an imaging sonar. The explosive is easier to identify when viewed from the side (left image) than when viewed from above (right image). This difference motivates active planning to identify the object.}
  \label{fig:mine}
\end{figure}

In addition to choosing the most informative views of the object, an autonomous vehicle is able to act adaptively by modifying its plan as new information from viewing the object becomes available. Consider an object of interest, such as an explosive, that has an identifiable feature on a particular side. If the vehicle receives a view that increases the likelihood of that object being in the frame, it would be advantageous to search for that identifiable feature to either exclude or confirm the identification of that object. A significant benefit from acting adaptively has been shown in the stochastic optimization and planning domains~\cite{golovin10COLT,dean08MOR}.

In this paper, we apply the above insights to active inspection in the underwater domain. This paper makes three main contributions. We
\begin{enumerate}

\item {\bf formalize} the active classification problem, combining classical work in sequential hypothesis testing with recent work in active learning,

\item {\bf analyze} the benefit of adaptivity, leading to an information theoretic heuristic for planning informative paths for active classification, and

\item {\bf apply and test} the approach to underwater classification in a simulated domain and using real-world data.
\end{enumerate}


\section{Related Work}
\label{sect:rel}

The problem of active classification is closely related to the classical problem of sequential hypothesis testing, where a sequence of noisy experiments are used to determine the nature of an unknown~\citep{wald45AMS}. This early work focussed on determining when to discontinue testing and make a final decision on the hypothesis. In classical sequential hypothesis testing, one performs a single experiment until the Bayes' risk is below a threshold. A key distinction between sequential hypothesis testing and active classification is that the type of experiment does not change in sequential testing. One of the first applications of sequential hypothesis testing to sensor placement applications was due to \citet{cameron90IJRR}. They discuss a Bayesian selection framework for identifying 2D images with multiple sensor placements. This work provides a foundation for the formulation discussed in the current paper, though it is limited to 2D images and does not discuss the use of salient features to determine informativeness.

The active classification problem can be seen as an instance of informative path planning~\citep{singh09JAIR}. Informative path planning optimizes the path of a robot to gain the maximal amount of information relative to some performance metric. It has been shown in several domains, including sensor placement~\citep{krause05UAI} and target search~\citep{hollinger09IJRR}, that many relevant metrics of informativeness satisfy the theoretical property of \textit{submodularity}. Submodularity is a rigorous characterization of the intuitive notion of diminishing returns that arises in many active planning application.

Recent advances in active learning have extended the property of submodularity to cases where the plan can be changed as new information is incorporated. The property of \textit{adaptive submodularity} was introduced by \citet{golovin10COLT}, which provides performance guarantees in many domains that require adaptive decision making. Their recent work examines these theoretical properties in the context of a sequential hypothesis testing problem with noisy observations~\citep{golovin10NIPS}. The idea of acting adaptively has also been examined in stochastic optimization and shown to provide increases in performance for stochastic covering, knapsack~\citep{dean08MOR}, and signal detection~\citep{naghshvar10ISIT}. To our knowledge these ideas have not been formally applied to robotics applications.


In the underwater inspection and surveying domains, there has been significant work in applying learning techniques to determine the nature of a marine environment. For example, \citet{steinberg10IROS} utilize Gaussian Mixture Models to classify marine habitats. They explain the need for adaptive classification, and the learning methods they develop help to facilitate that goal. In addition, there has been limited work in utilizing multiple views to classify underwater mines. In some work, an assumption is made that all views provide the same amount of information~\citep{williams09TIP}, and in other work the focus is on designing high-level mission planning capabilities to ensure coverage of the sea floor~\citep{williams10ICRA}. To our knowledge, the problem of determining a path that maximizes classification accuracy based on viewpoints of differing informativeness has not been studied in the underwater inspection domain.

The problem of active multi-view recognition has been studied extensively for computer vision applications~\citep{sipe02PAMI,denzler02PAMI,schiele98ICCV}, including the use of depth maps in medical imagery~\citep{zhou03ICCV}. \citet{ma10ICRA} also provide a recent application of active planning for simultaneous pose estimation and recognition of a moving object using a mobile robot. While different forms of information gain play a critical role in these prior works, a key distinction in our work is the notion of adaptivity. In active classification problems, selecting the next best observation, or even an initial ordering of informative observations, may not result in overall performance optimization. It is in this regard that we provide new analysis of the benefit of adaptivity and make connections to performance guarantees in submodular optimization and active learning. Our analysis is complementary to prior computer vision work and could potentially be extended to many of these alternative frameworks.

\section{Problem Formulation}
\label{sect:prob}

We will now formulate the active classification problem within the sequential hypothesis testing framework~\citep{wald45AMS}. The goal is to determine the class of an unknown object given a set of $N$ possibilities $\mathcal{H} = \{h_1,\ldots,h_N\}$. Let $H$ be a random variable equal to the true class of the object. In the simplest case, a binary classification task is considered (e.g., $H = h_0$ denotes an object of interest and $H = h_1$ denotes the lack of such an object). We can observe the object from a set of possible locations $\mathcal{L} = \{L_1, \ldots, L_M\}$, where the locations themselves are not informative.\footnote{We formulate the problem for the case of discrete locations. If continuous locations are available, an interpolation function can be used to estimate the informativeness of a location based on the discrete training data (see Section~\ref{sect:res}).} There is a cost of moving from location $L_i$ to location $L_j$, which we denote as $d_{ij}$. In robotics applications, this cost is determined by the kinematics of the vehicle and the dynamics of both the vehicle and environment.

A set of features $\mathcal{F} = \{F_1,\ldots,F_K\}$ is also given that distinguishes between objects. Each feature $F_i$ is a random variable, which may take on some values (e.g., binary, discrete, or continuous). Given one or more template images for each class $n$, we can calculate a function $G(L) : \mathcal{L} \rightarrow \mathcal{F}$ mapping viewing location $L$ to the features for which realizations will be observed from that viewing location. In general, this mapping may be stochastic and dependent on the class. The mapping from location to features is a key characteristic of robotics applications that differentiates our problem from the more common problem where the features can be observed directly~\citep{golovin10NIPS}. Figure~\ref{fig:gmodel} shows a graphical model of the resulting problem.

We assume knowledge of a prior distribution for each class $P(H)$, as well as a conditional probability for each feature given the class $P(F_k \mid H)$. The conditional distribution represents the probability of each feature taking on each of its possible values given the class. These probabilities can be estimated via training data. The features that have been viewed evolve as the robot moves from location to location. At a given time $t$, the robot is at location $L(t)$, and we observe realizations of some new features $\mathcal{F}_t \subset \mathcal{F}$. Let us define $\mathcal{F}_{1:t} \colonequals \cup_{i=1}^t \mathcal{F}_i$ as the features observed up until time $t$. If we assume that the features are conditionally independent given the class, we can calculate a distribution $b(t) = \{b_1, \ldots, b_N\}$ using standard recursive Bayesian inference~\citep{thrunbook}:

\begin{align}
 b(t) & \colonequals P(H \mid \mathcal{F}_{1:t})\\
  &=  \eta \ b(t-1) \prod_{F \in \mathcal{F}_t} P(F \mid H),
\end{align}
where $\eta$ is a normalizing constant.

\begin{figure}[tb]
  \centering
  \includegraphics[width=2.75in]{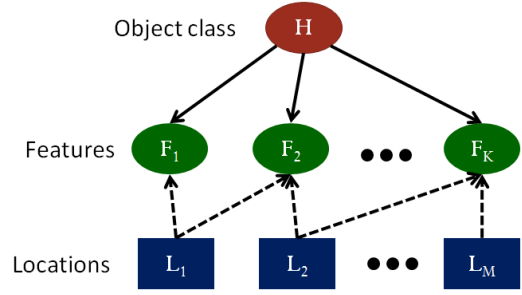}
  \caption{Graphical model of an active classification problem. The goal is to determine the value of the hypothesis $H$ by observing a subset of features $F_1,\ldots,F_K$. The features cannot be viewed directly, but must instead be viewed by moving to some locations $L_1,\ldots,L_M$. The solid lines denote stochastic dependence, and the dashed lines denote which features can be viewed by visiting each location. Dependencies between the features could also exist, which would break the conditional independence assumption.}
  \label{fig:gmodel}
\end{figure}

The goal is to find a policy $\pi$ that takes a belief distribution $b(t)$, current location $L(t)$, and observation history $\mathcal{F}_{1:t}$ and determines the next location from which to view the object. Note that the dependence on the observation history and current distribution allows the policy to be adaptive as new information becomes available.

\subsection{Noiseless Case}

Ideally, we would like to run the policy until we know the object's class. If the observations do not contain any noise, this goal is reachable. For each hypothesis $h$, a policy $\pi$ will have a cost $c(\pi,h)$ associated with the locations the policy visits. We define the expected cost of this policy relative to a distribution on hypothesis $P(H)$ as:

\begin{equation}
 c(\pi) \colonequals \mathbb{E}_H [ c(\pi,h) ]
\end{equation}

This equation represents the expected cost for the policy $\pi$. For the noiseless case, we assume that each hypothesis $h$ has an associated vector $V_h = [f_1, \ldots, f_K]$ of feature values that \textit{always} occur for that hypothesis. As a result, $P(F_1, \ldots, F_K \mid H)$ only takes on the values of one or zero. An incomplete feature vector $V$ is said to be consistent with a hypothesis $h$ if for all $f \in V$ we have $f \in V_h$.

Without observation noise, we may fully determine the hypothesis by observing some features (in some cases all features). Let $\mathcal{V}(V)$ represent the number of classes that are consistent with partial feature vector $V$ (also referred to in prior work as the version space~\cite{golovin10NIPS}). Let $V(\pi,h)$ be the feature vector that results from executing policy $\pi$ with hypothesis $h$. The optimal policy is now the one that optimizes the equation below:

\begin{equation}
 \pi^* = \argmin_{\pi} c(\pi) \mbox{ s.t. } \mathcal{V}(V(\pi,h)) = 1 \mbox{ for all } h \in \mathcal{H}
\end{equation}

Even in the noiseless case, there may be insufficient features to determine the exact class of the unknown object. In these cases, the goal would be to observe the fewest number of features that reduce the number of consistent classes as much as if all features were observed.

\subsection{Noisy Observations}

When the observations are noisy, it will likely be impossible to determine the class of an unknown object with certainty. However, as in the decision theory literature, we minimize the expected loss (also known as the Bayes' risk~\cite{wald45AMS}) of the final classification decision. We will now formulate the problem of minimizing Bayes' risk for the case of noisy observations. With noisy observations, $P(F \mid H)$ takes on values other than one or zero. As a result, there is no longer a deterministic vector $V_{h}$ associated with each hypothesis, and typically we cannot uniquely determine the hypothesis even by observing all features.

In the noisy observation case, we can generate a policy that minimizes a loss function $l(d,h)$ associated with making a decision $d$ for that object (i.e., deciding on the object's class). For instance, if the object is an explosive, a false negative could incur a very high cost, but a false positive would be a lower cost. If we select the class with maximum \textit{a posteriori} probability after running a policy $\pi$, we can calculate the expected loss for running that policy to completion:

\begin{equation}
l(\pi) \colonequals \mathbb{E}_H [l(d,h) \mid \pi]
\end{equation}

Let $\tau$ be an acceptable threshold on expected loss. A natural goal is to incur the lowest cost and achieve the same expected loss. The resulting optimization problem is given below:

\begin{equation}
 \pi^* = \argmin_{\pi} c(\pi) \mbox{ s.t. } l(\pi) \leq \tau
\end{equation}

\section{Proposed Solution}
\label{sect:alg}

The goal is to optimize the expected loss for a policy $\pi$. The expected loss is a function of the final belief $b(T)$, which represents $P(H \mid \mathcal{F}_{1:T})$. Calculating this loss on an infinite horizon would require examining an exponential number of paths in the horizon length. To make the computation feasible, we can use the truncated expected loss:

\begin{equation}
 \pi^* = \argmin_{\pi \in \Pi(1:T)} \mathbb{E}_H [ l(d,h) \mid \pi(1 : T) ]
\label{eq:trunc}
\end{equation}

A related measure of the quality of $b(T)$ is the \textit{information gain} of the class given the features observed $IG(H; \mathcal{F}_{1:T}) = \mathbb{H}(H) - \mathbb{H}(H \mid \mathcal{F}_{1:T})$, where $\mathbb{H}$ is the entropy. We will motivate the use of information gain further in Section~\ref{sect:theory}. A heuristic for solving the active classification problem using information gain can be formulated as below:

\begin{equation}
 \pi^* = \argmax_{\pi \in \Pi(1:T)} \mathbb{E}_H [ IG(H; \mathcal{F}_{1:T}) \mid \pi(1 : T) ]  ,
\label{eq:opt}
\end{equation}
where $\Pi(1:T)$ is the set of all possible policies truncated at time $T$. If this optimization is performed on the receding horizon, it allows for adaptive decision making with a finite lookahead. The path costs can be implicitly incorporated by looking ahead to a ``cost horizon." This approach has been shown to perform well in similar information gathering domains~\citep{hollinger09IJRR}.

For some loss functions, the information gain objective is equivalent to minimizing the Bayes' risk. One such function for the binary hypothesis case is the standard $0/1$ loss, where cost of one is incurred for an incorrect classification, and no cost is incurred for a correct classification.

\section{Theoretical Analysis}
\label{sect:theory}

We now relate the active classification problem to recent advances in active learning theory that allow us to analyze the performance of both non-adaptive and adaptive policies. Active classification falls into a class of informative path planning problems~\citep{singh09JAIR}. Given some potential locations to make observations, the informative path planning problem is to maximize a function $F(A)$, where $A = \{L_1,L_2,\ldots,L_T\}$ is a set of locations visited by the vehicle up to an end time $T$. In most cases, the sets of possible locations to visit are constrained by obstacles, vehicle kinematics, or other factors. For the active classification problem, $F(A)  = -\mathbb{E}_H [l(d,h) \mid A]$, the negative expected loss after observing along path $A$.

\subsection{Performance Guarantees}

A \textit{non-adaptive policy} is one that generates an ordering of locations to view and does not change that ordering as features are observed. The non-adaptive policy will typically be easier to compute and implement, since it can potentially be computed offline and run without modification. Performance guarantees in the non-adaptive informative path planning domain are mainly dependent on the objective function (i.e., the informativeness of the views) being non-decreasing and submodular on the ground set of possible views. A set function is non-decreasing if the objective never decreases by observing more locations in the environment. A set function is submodular if it satisfies the notion of diminishing returns (see \citet{singh09JAIR} for a formal definition).

Information gain has been shown to be both non-decreasing and submodular if the observations are conditionally independent given the class~\citep{krause05UAI}, as is assumed in this paper (see Section~\ref{sect:prob}). Thus, if the loss function is equivalent to information gain (e.g., 0/1 loss with binary hypotheses), then the active classification problem optimizes a non-decreasing, submodular function. Let $A^{IG}$ be the set of locations visited by the information gain heuristic with a one-step lookahead. For non-adaptive policies without path constraints (e.g., when traversal costs between locations are negligible compared to observation cost), we have the following performance guarantee: $F(A^{IG}) \geq (1-1/e) F(A^{opt})$~\citep{krause05UAI}.

When path constraints are considered, the recursive greedy algorithm, a modification of greedy planning that examines all possible middle locations while constructing the path, can be utilized to generate a path $A^{rg}$~\citep{singh09JAIR}. Recursive greedy provides a performance guarantee of $F(A^{rg}) \geq F(A^{opt}) / \log(|A^{opt}|)$, where $|A^{opt}|$ is the number of location visited on the optimal path. However, the recursive greedy algorithm requires pseudo-polynomial computation, which makes it infeasible for some application domains. To our knowledge, the development of a fully polynomial algorithm with performance guarantees in informative path planning domains with path constraints is still an open problem. Hence, we utilize a one-step heuristic in our experiments in Section~\ref{sect:res}.

The performance guarantees described above do not directly apply to adaptive policies. An \textit{adaptive} policy is one that determines the next location to select based on the observations at the previously viewed locations. Rather than a strict ordering of locations, the resulting policy is a tree of locations that branches on the observation history from the past locations. As noted earlier, the concept of adaptive submodularity~\citep{golovin10COLT} allows for some performance guarantees to extend to adaptive policies as well. When the observations are noiseless, the information gain heuristic satisfies the property of adaptive submodularity. This result leads to a performance guarantee on the cost of the one-step information gain adaptive policies in sequential hypothesis testing domains without path constraints: $c(\pi_{IG}) \leq c(\pi^{opt}) (\ln(1/p_{min}) + 1)$, where $p_{min} \colonequals \min_{h \in \mathcal{H}} P(h)$. When noisy observation are considered, a reformulation of the problem is required to provide performance guarantees (i.e., information gain is not adaptive submodular). However, \citet{golovin10NIPS} show that the related Equivalence Class Determination Problem (ECDP) optimizes an adaptive submodular objective function and yields a similar logarithmic performance guarantee. The direct application of ECDP to active classification is left for future work.

\subsection{Benefit of Adaptivity}

We now examine the benefit of adaptive selection of locations in the active classification problem. As described above, the non-adaptive policy will typically be easier to compute and implement, but the adaptive policy could potentially perform better. A natural question is whether we can quantify the amount of benefit to be gained from an adaptive policy for a given problem. To begin our analysis of adaptivity, we consider the problem of minimizing the expected cost of observation subject to a hard constraint on loss\footnote{Note that the related problem of minimizing expected loss subject to a hard constraint on budget is also relevant. While similar examples show that there is a benefit to acting adaptively in this case, we defer detailed analysis to future work.}:

\begin{problem}
Given hypotheses $\mathcal{H} = \{h_1, h_2, \ldots, h_N\}$, features $\mathcal{F} = \{F_1, F_2, \ldots, F_K\}$, locations $\mathcal{L} = \{L_1, \ldots, L_M\}$, costs $c(L_i,L_j) = d_{ij}$ for observing location $i$ when at location $j$, and a loss function defined as $l(d,h)$ for selecting hypothesis $d$ when the true hypothesis is $h$. We wish to select a policy $\pi$ such that:

\begin{equation}
 \pi^* = \argmin_{\pi} c(\pi) \mbox{ s.t. } l(\pi) \leq \tau,
\end{equation}
where $l(\pi) \colonequals \mathbb{E}_H [l(d,h) \mid \pi]$,  $c(\pi) \colonequals \mathbb{E}_H [c(\pi,h)]$, and $\tau$ is a scalar threshold.
\label{prob:loss}
\end{problem}

We now show that the optimal non-adaptive policy can require exponentially higher cost than an adaptive policy for an instance of this problem:

\begin{theorem}
Let $\pi_{adapt}$ be the optimal adaptive policy, and $\pi_{non-adapt}$ be the optimal non-adaptive policy. There is an instance of Problem~\ref{prob:loss} where $c(\pi_{adapt}) = \log(N)$ and $c(\pi_{non-adapt}) = N-1$, where is $N$ is the number of hypotheses.
\label{prop:loss}
\end{theorem}
\begin{proof}
We adopt a proof by construction. Let $\tau = 0$, i.e., the required expected loss is zero. Let the features be observed directly through the corresponding locations (i.e., $G(L_i) = F_i$ and $M = K$). Let there be $N$ hypotheses and $M = N-1$ features. Assign a cost $c(F) = 1$ for all features. The loss $l(d,h) = 1$ for all $d \neq h$ and $l(d,h) = 0$ for $d = h$.

Let $P(h) > 0$ for all $h \in \mathcal{H}$. Let $P(F_1 | h_i) = 1$ for all $i \in \{1,\ldots,N/2\}$ and $P(F_1 | h_i) = 0$ for all $i \in \{N/2+1, N\}$. That is, feature $F_1$ is capable of deterministically differentiating between the first half and second half of the hypotheses. $P(F_2 | h_i) = 1$ for all $i \in \{1, N/4\}$, $P(F_3 | h_i) = 0$ for all $i \in \{N/4+1, N/2\}$, and $P(F_2 | h_i) = 1/2$ for all $i \in \{N/2+1, N\}$. That is, feature $F_2$ is capable of deterministically differentiating between the first fourth and second fourth of the hypothesis space but gives no information about the rest of the hypotheses. Similarly, define $P(F_3 | h_i) = 1$ for all $i \in \{N/2+1, 3N/4\}$, $P(F_3 | h_i) = 0$ for all $i \in \{3N/4+1, N\}$, and $P(F_2 | h_i) = 1/2$ for all $i \in \{1, N/2\}$. The remaining features are defined that differentiate progressively smaller sets of hypotheses until each feature differentiates between two hypotheses.

The adaptive policy will select $F_1$ first. If $F_1$ is realized positive, it will select $F_2$. If $F_1$ is realized negative, it will select $F_3$. It will continue to do a binary search until $\log(N)$ features are selected. The true hypothesis will now be known, resulting in zero expected loss. In contrast, the non-adaptive policy must select all $N-1$ features to ensure realizing the true hypothesis and reducing the expected loss to zero. \qed
\end{proof}

The adaptivity analysis in Theorem~\ref{prop:loss} requires multiple hypotheses, and the potential benefit of adaptivity increases as the number of hypotheses increases. For the two hypothesis case, however, the benefit of adaptivity may be very small. In the binary examples we have examined, all cases showed little or no benefit from adaptivity. Furthermore, if there is a strict ordering on the informativeness of the viewing locations independent of the current distribution on the hypotheses, we conjecture that the benefit of acting adaptively will be zero~\citep{naghshvar10ISIT}.



\section{Implementation and Experiments}
\label{sect:res}

In this section, we examine the active classification problem experimentally through the use of both synthetic images and data from imaging sonar during ship hull inspection. The results confirm the benefit of active view selection in these application domains as well as the benefit of adaptivity when more than two hypotheses are considered. For all experiments, we assume a simple 0/1 loss model, where a cost of one is incurred for a false classification, and a cost of zero is incurred for a correct classification.

\subsection{Synthetic Images}

The goal of our first experiments is to differentiate between possible polyhedra using depth maps from different views. The relevance of polyhedra recognition to underwater inspection is direct, as explosive devices are often cubic or pyramidal in shape~\citep{dobeck02SPIE}. This is a particularly challenging active recognition problem due to similarities and symmetries between polyhedra. These experiments are designed to (1) demonstrate the benefit of selecting the views with the highest potential for information about the unknown object, and (2) examine the benefit of acting adaptively when multiple possible objects are examined.

To identify the polyhedra, we utilize salient features extracted from the synthetic depth map. Training images were created from 24 different viewpoints around the objects, and the OpenCV~\cite{opencv} SURF feature extractor~\cite{bay08CVIU} was used to extract features for the different object and viewpoints viewpoints. Noisy test images were then created with Gaussian white noise ($\sigma = 0.25 \ m$).

\subsubsection{Two objects}

The intuition is that it will be easier to identify the object in some viewpoints than in others, due to the presence of additional salient features. Figure~\ref{fig:surf} shows SURF features and correlations for a tetrahedron and cube viewed from the face and vertex. The number of SURF features and correlations is greater for viewing the vertices when compared to viewing the faces. Particularly for the cube, viewing the face provides few correlations and little information about the object class.

\begin{figure}[bt]
  \centering
  \subfigure{\includegraphics[height=0.7in]{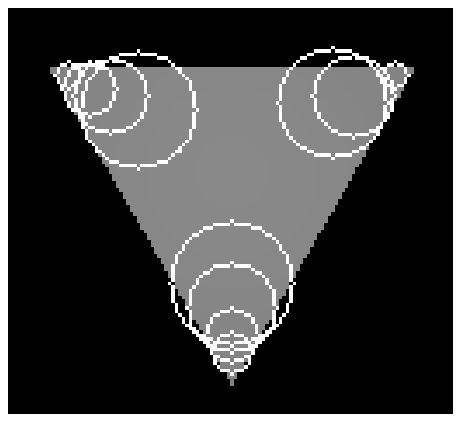}}
  \subfigure{\includegraphics[height=0.7in]{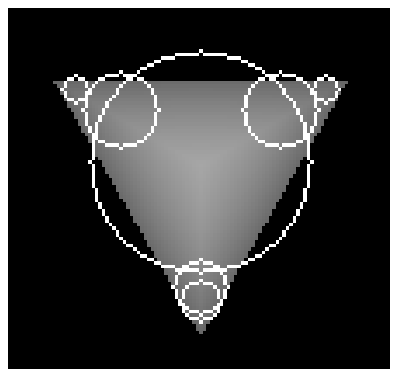}}
  \subfigure{\includegraphics[width=0.7in]{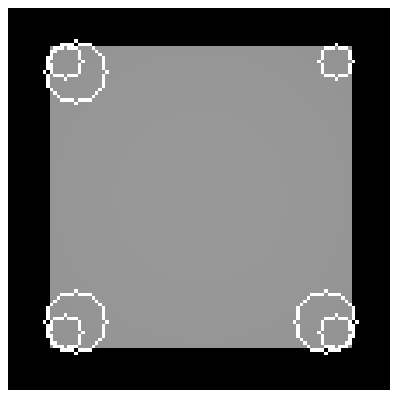}}
  \subfigure{\includegraphics[height=0.7in]{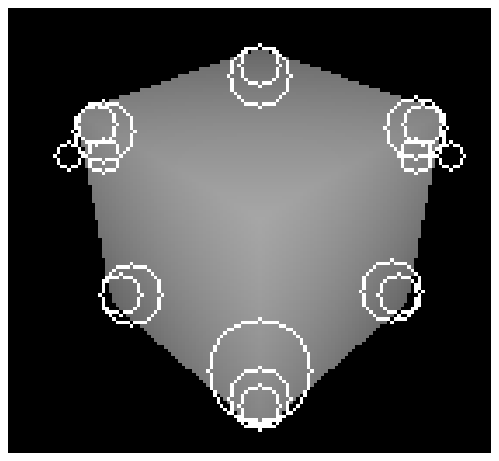}}\\
  \subfigure{\includegraphics[height=1.5in]{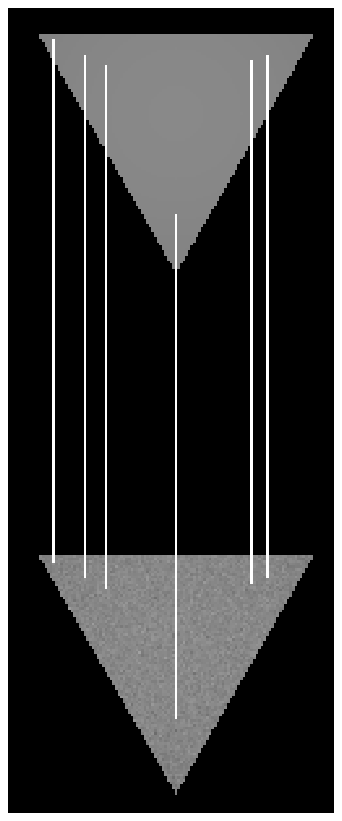}}
  \subfigure{\includegraphics[height=1.5in]{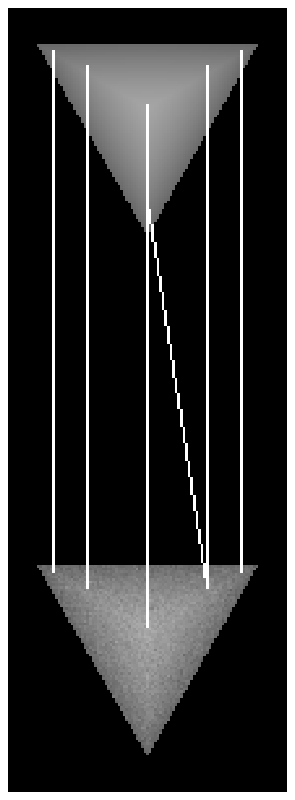}}
  \subfigure{\includegraphics[height=1.5in]{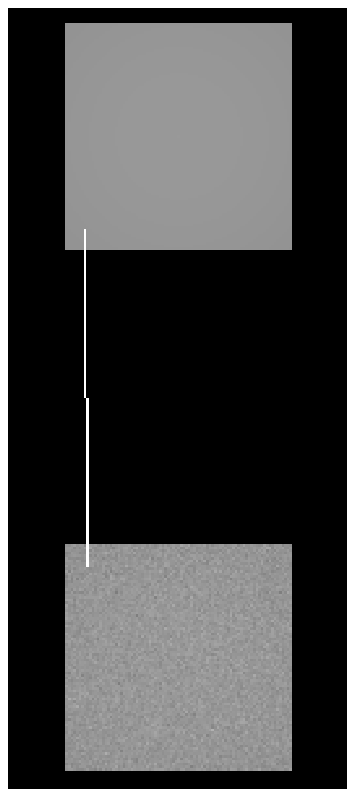}}
  \subfigure{\includegraphics[height=1.5in]{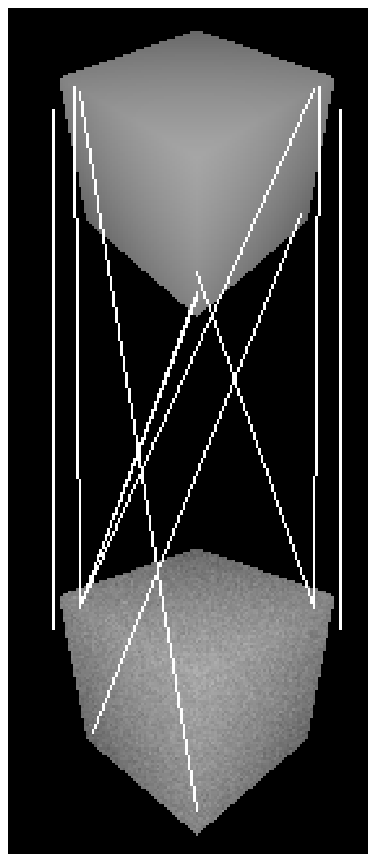}}
  \caption{Top: SURF features extracted from viewing depth maps of tetrahedron and cube faces and vertices. Bottom: SURF feature correlations when compared to noisy depth maps.}
    \label{fig:surf}
\end{figure}

For quantitative analysis, we now compare informative view selection to random view selection on the synthetic depth map data from a cube and tetrahedron. The information gain of each view was calculated based on the number of expected salient features corresponding to the true object minus the expected number of false correspondences. This calculation requires comparing all views to the corresponding views of each other object ($O(N^2)$ computation in the number of hypotheses). After the cross-correlations were computed, planning was completed in milliseconds. To apply adaptive view selection, we calculate the information gain from the current distribution over the features, which changes as new views are observed.

In these experiments, path constraints are ignored, though the view ordering could easily be used to generate a feasible path on the finite horizon. Figure~\ref{fig:synth} shows results comparing the information gain heuristic with random view orderings. Utilizing the information gain heuristic to determine the most informative views leads to as much as a 35\% increase in the number of correct feature correspondences with limited views. Adaptive view selection does not provide much benefit over the non-adaptive technique, as expected from the small adaptivity gap in the binary hypothesis case (see Section~\ref{sect:theory}). Note that, for comparison, only 24 views are considered, and all methods will provide the same performance after seeing all these views.


\begin{figure}[bt]
  \centering
  \subfigure{\includegraphics[width=2.25in]{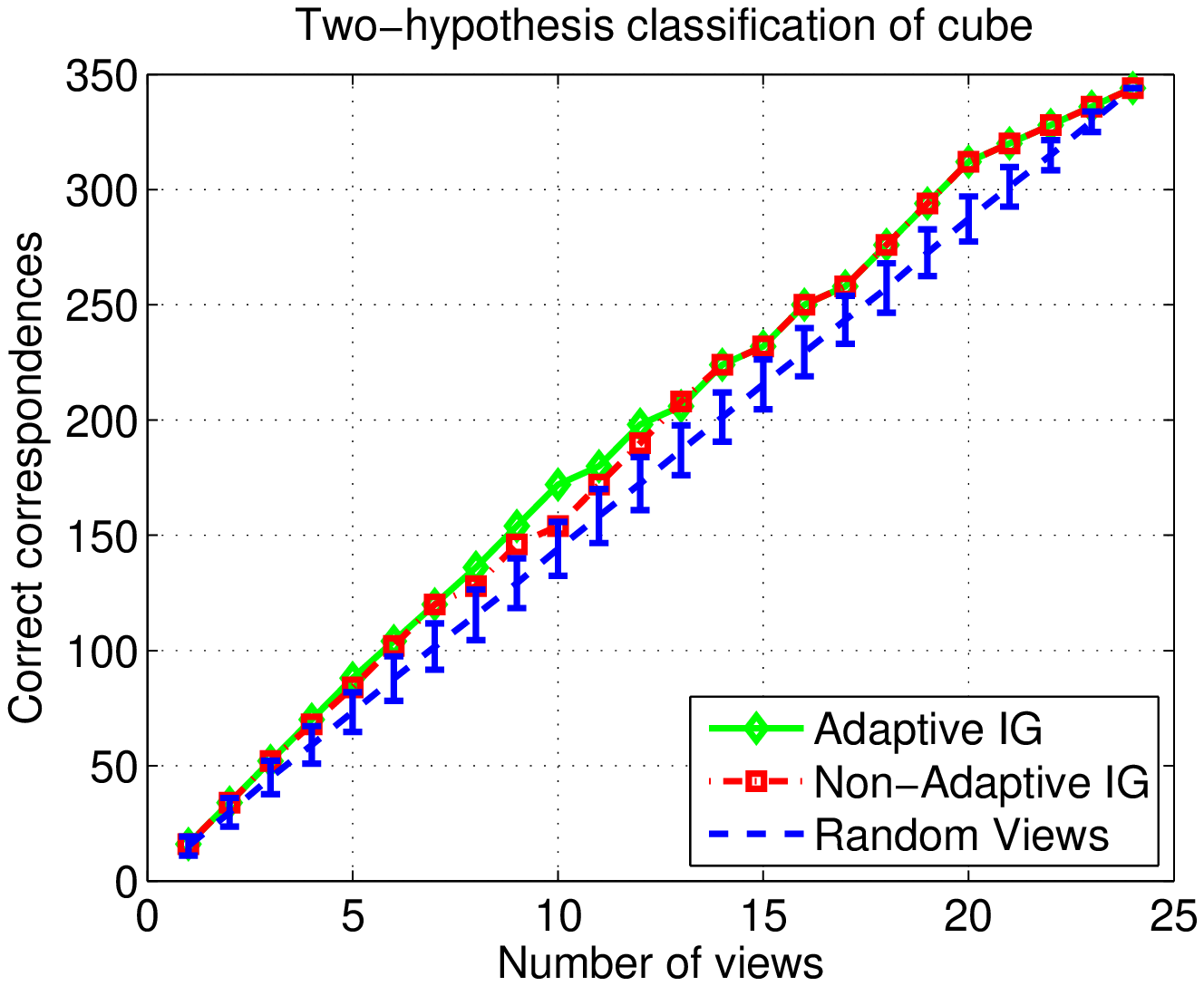}}
  \subfigure{\includegraphics[width=2.25in]{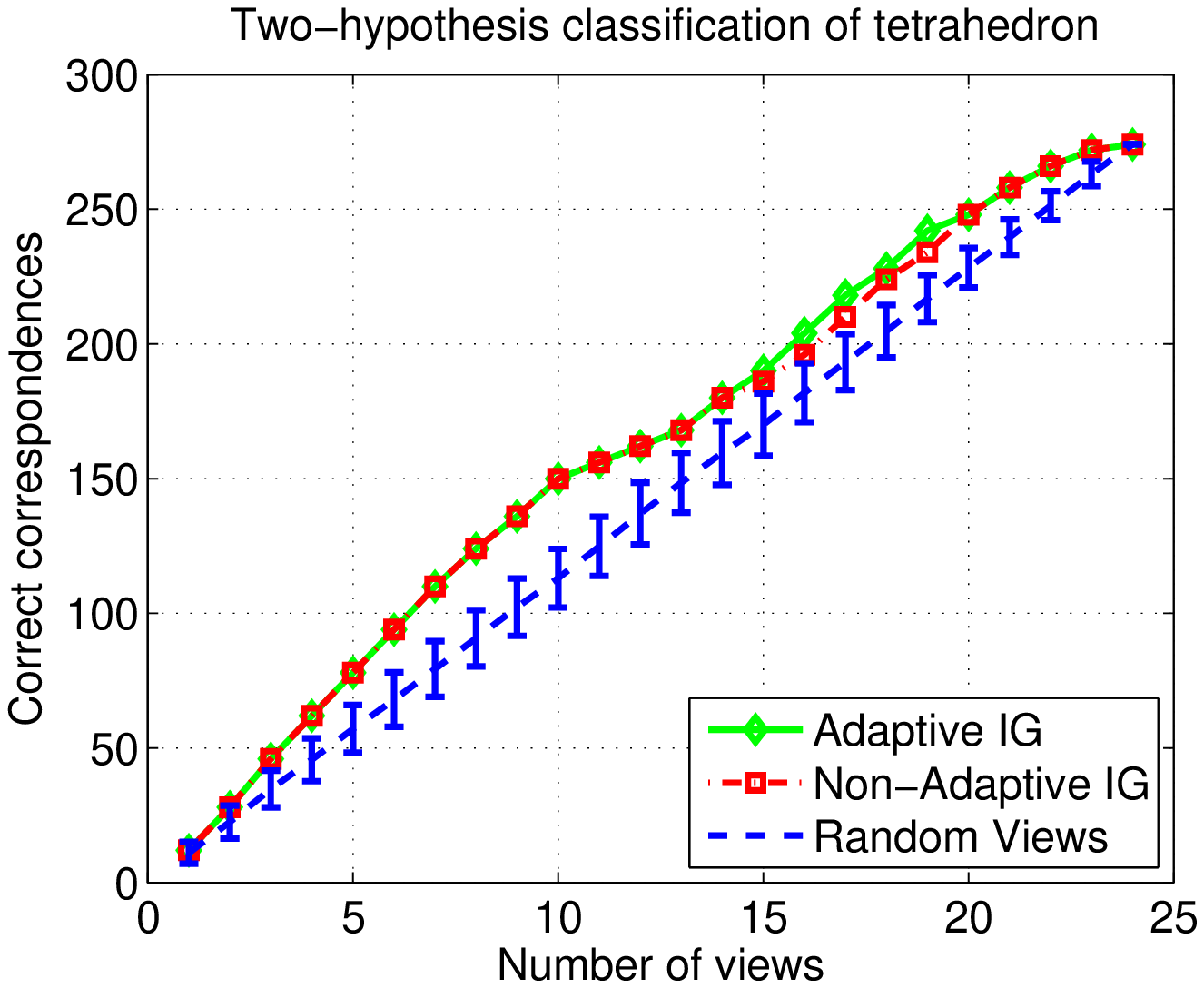}}
  \caption{Multi-view classification experiments with synthetic images of a cube and tetrahedron viewed from 24 different angles (best viewed in color). Utilizing the expected information gain of the next view improves the number of SURF feature correspondences when limited views are used. Random view results are averaged over 100 orderings; error bars are one standard deviation.}
  \label{fig:synth}
\end{figure}


\subsubsection{Multiple objects}

The benefit of active classification is now examined for cases where more than two object classes are considered. In addition to the cube and tetrahedron, we include training images of the icosahedron, octahedron, and dodecahedron as possible object classes. The theoretical analysis in Section~\ref{sect:theory} suggests that acting adaptively should improve performance for the multi-hypothesis problem. Figure~\ref{fig:multisynth} shows results for classifying the cube and tetrahedron when additional hypotheses are considered for the other three platonic solids. The adaptive policy outperforms both random view selection and the non-adaptive policy the majority of the time. The difference is particularly significant for the tetrahedron. Note that the dominance of the adaptive policy is not true at all data points. These results suggest that adding additional hypotheses in some cases reduces the performance of active view selection.


\begin{figure}[bt]
  \centering
  \subfigure{\includegraphics[width=2.25in]{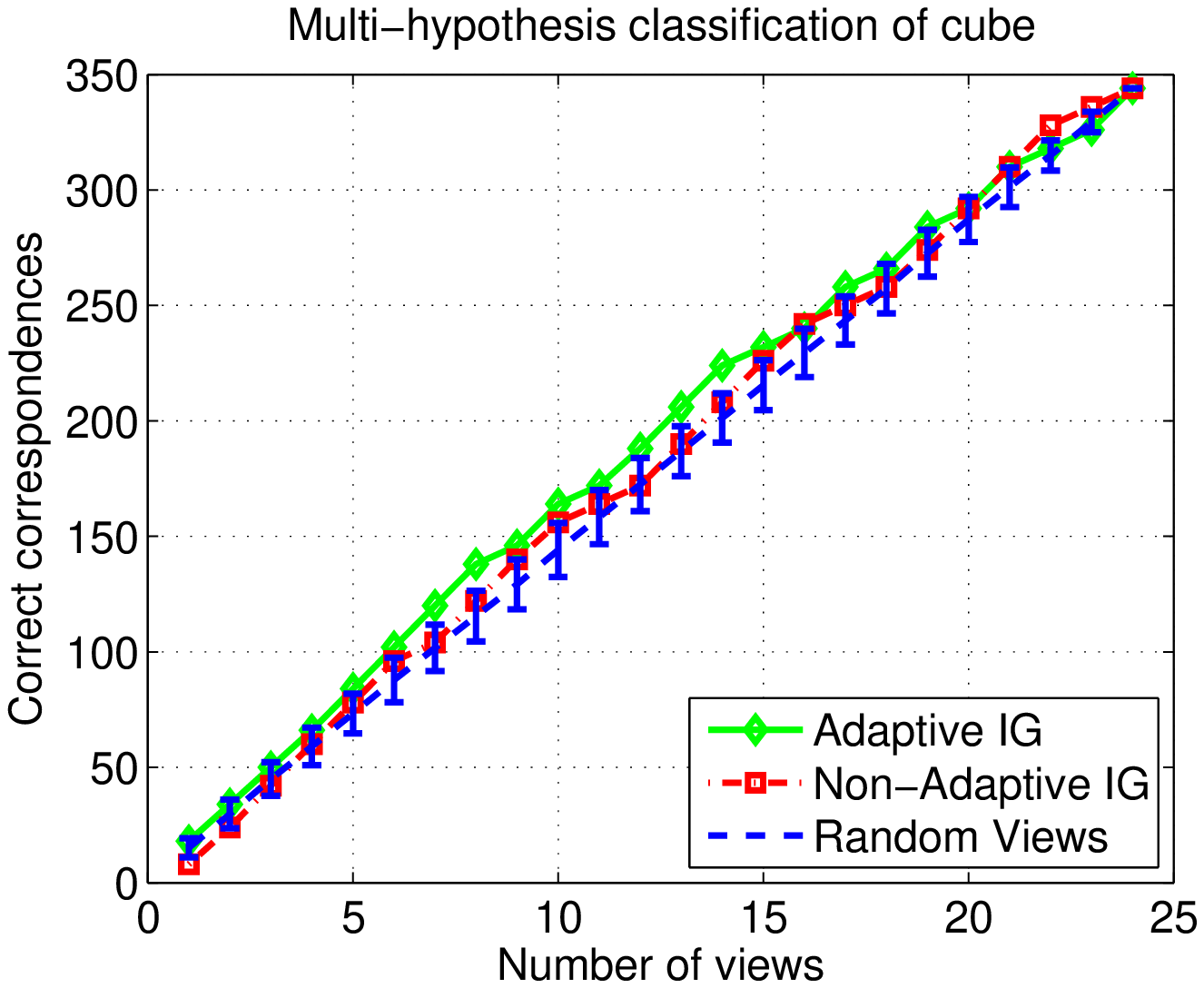}}
  \subfigure{\includegraphics[width=2.25in]{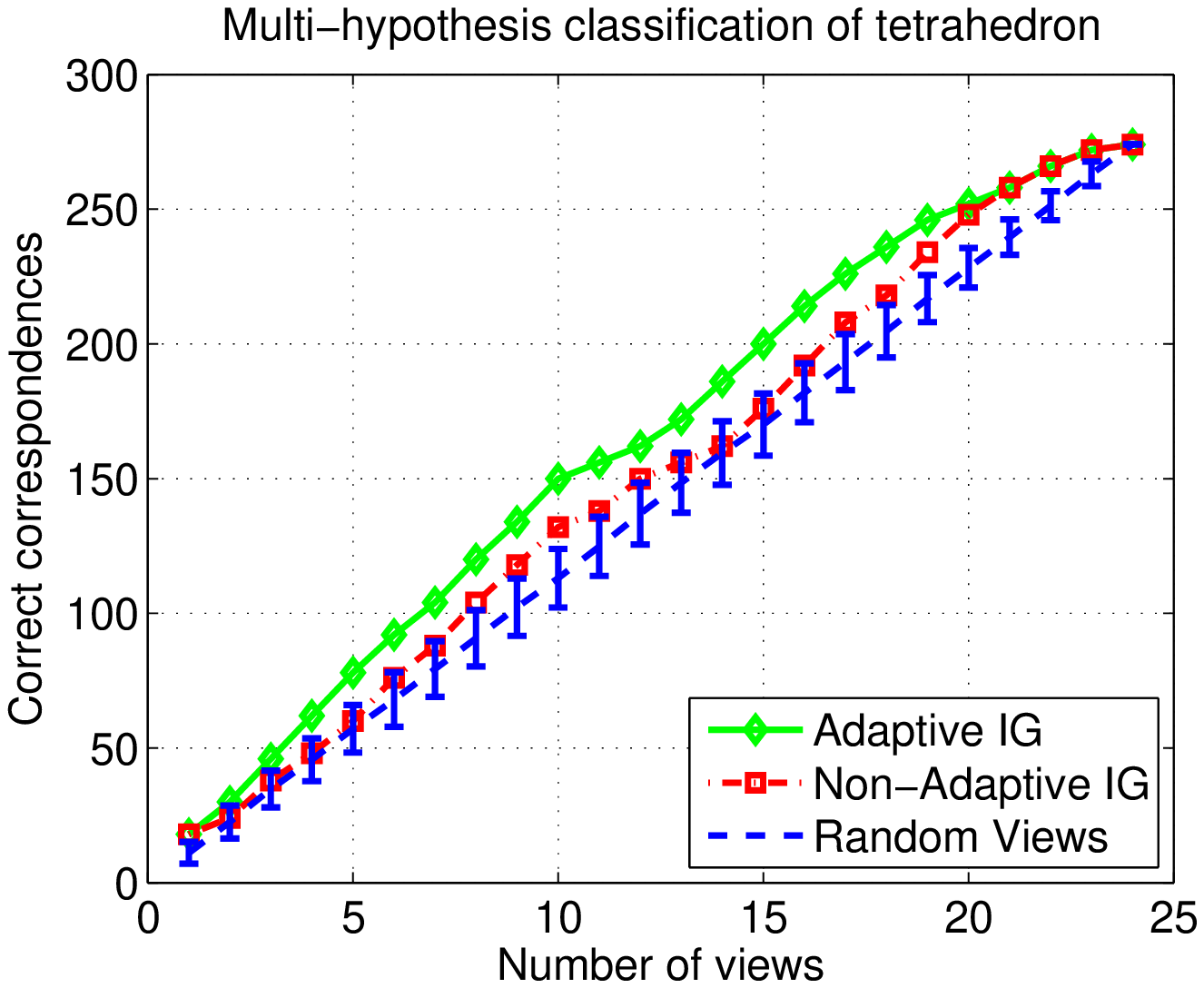}}
  \caption{Classification experiments with synthetic images of the five platonic solids (best viewed in color). The results for a cube and tetrahedron test object are shown. Adaptively selecting the most informative views based on past information tends to improve classification accuracy, and acting adaptively increases this benefit. Random view results are averaged over 100 random orderings; error bars are one standard deviation.}
  \label{fig:multisynth}
\end{figure}

\subsection{Imaging Sonar Data}

To examine the benefit of active classification on real-world data, we ran experiments on imaging sonar depth maps taken from a ship hull inspection with an underwater vehicle. The goal is to determine whether an explosive has been placed on the ship hull. The explosive appears as a small patch of bright pixels on the imaging sonar depth map. Since the sonar data is not dense enough to provide salient features, we take a simpler approach of using the brightness of the pixels as the feature base. A brightness threshold was learned by minimizing the number of misclassified pixels in labeled data. The performance metric is the total number of pixels correctly classified as part of the explosive device. We utilize this metric because images with a large number of corresponding pixels may provide additional information during post-processing or to a human operator.

A separate test set was held out of the labeled set to determine if the most informative views could be predicted using the learned threshold and expected view quality. There were 100 frames in the training and 75 frames in the test set. The training and test frames were from different trajectories, but with the same background. The frame rate was approximately 2 fps. The information gain in these experiments was calculated based on the expected number of pixels corresponding to the explosive in a given view, which was found using an average of the hand-labeled pixels in the training set images weighted by their distance (using data from a DVL sensor). A squared exponential weighting was used.

Figure~\ref{fig:exp} shows the results of running the information gain approach versus random views. We also compare to the initial (very poor) view ordering from the data as well as two simple ordering methods: sorting the views based on minimum distance to the object and sorting based on the maximum angle of view (see Figure~\ref{fig:mine} for the intuition behind this method). The results show that actively choosing the views with the highest expected information improves classification performance. For example, choosing informative views reduces the number of views for 15 correct pixel identifications by nearly 80\% versus random selection (from 38 views to 8 views).

\begin{figure}[bt]
  \centering
  \subfigure{\includegraphics[width=3.0in]{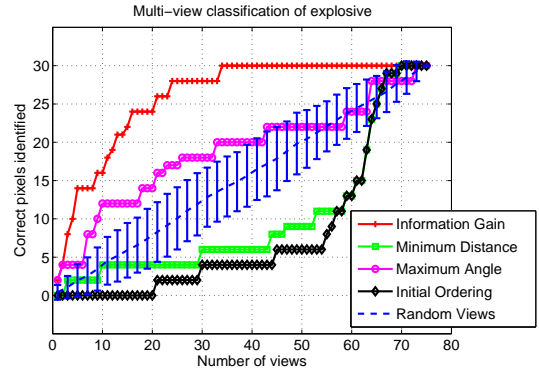}}
  \caption{Multi-view classification experiments with imaging sonar identifying an explosive on a ship hull. With limited views, utilizing information gain leads to a larger number of pixels correctly identified as part of the object of interest. Random view results are averaged over 100 random orderings; error bars are one standard deviation.}
  \label{fig:exp}
\end{figure}

For visual reference, Figure~\ref{fig:imaging} shows images of decreasing expected pixel classifications. Intuitively, the images where the explosive stands out from the background should provide the most information. Despite some incorrect predictions, it is clearly beneficial to examine those viewpoints predicted to be informative. It should be noted that the informativeness of the images depends on the quality of the low-level sonar processing. With perfect low-level data processing, all images may have high informativeness, which would reduce the benefit of active classification.


\begin{figure}[bt]
  \centering
  \subfigure[Exp. gain: 3.7]{\includegraphics[width=1.1in]{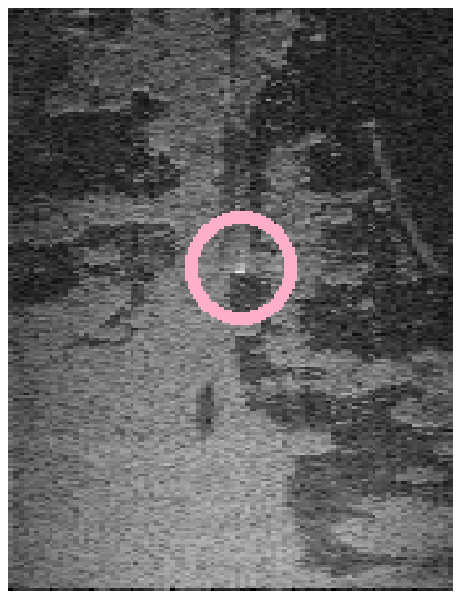}}
  \subfigure[Exp. gain: 2.1]{\includegraphics[width=1.1in]{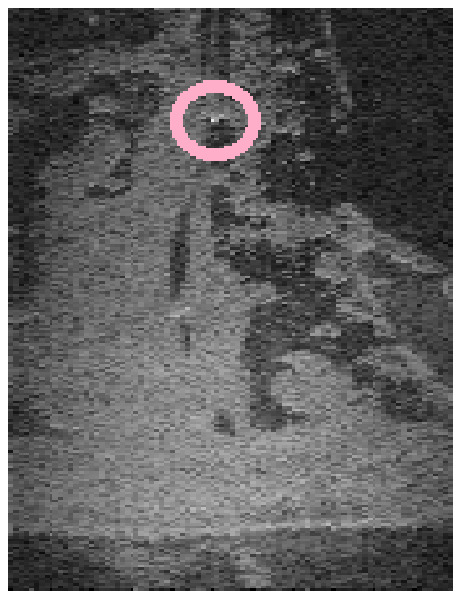}}
  \subfigure[Exp. gain: 1.5]{\includegraphics[width=1.1in]{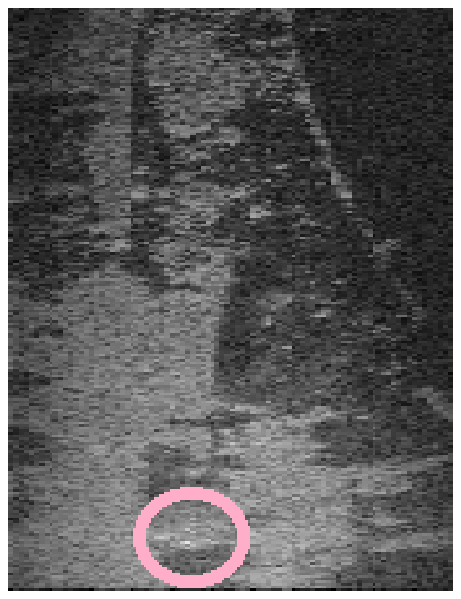}}
  \subfigure[Exp. gain: 0.8]{\includegraphics[width=1.1in]{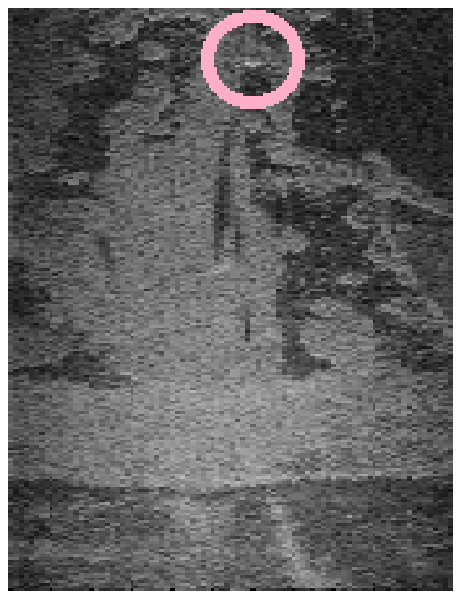}}
  \caption{Imaging sonar depth maps of an explosive device (circled) placed on a ship's hull. The depth maps are ordered based on the expected number of pixels in the image corresponding to a possible explosive. Note that the explosive is easy to identify in image (a), more difficult to identify in image (b), and very difficult to identify in image (c). Image (d) is expected to be a low information view, when in fact the explosive is relatively easy to identify.}
  \label{fig:imaging}
\end{figure}

\section{Conclusions and Future Work}
\label{sect:conc}

This paper has shown that actively choosing informed views improves performance for inspection tasks in the example underwater domain. The experimental results demonstrate that depth map information can be utilized to recognize objects of interest, and that (compared to passive methods) up to 80\% fewer views need to be examined if the views are chosen based on their expected information content. In addition, acting adaptively by re-evaluating the most informed views as new information becomes available leads to improvement when more than two classes are considered. These results are consistent with theoretical analysis of the benefit of adaptivity.

Future work includes further theoretical analysis of possible performance guarantees, particularly in the case of path constraints. In addition, the results in this paper utilize features for classification. Recent work in featureless classification through the use of point clouds would benefit from active classification methods as well. Finally, the analysis in this paper has applications beyond underwater inspection. Tasks such as ecological monitoring, reconnaissance, and surveillance are just a few domains that would benefit from active planning for the most informed views. Through better control of the information we receive, we can improve the understanding of the world that we gain from robotic perception.

\begin{acknowledgement}
The authors gratefully acknowledge Franz Hover and Brendan Englot at MIT for imaging sonar data and technical support while processing the data. Further thanks to Hördur Heidarsson at USC for assistance with data collection.
\end{acknowledgement}

\bibliographystyle{plainnat}
\bibliography{references}

\end{document}